\documentclass{svproc}
\usepackage{graphicx}%
\usepackage{multirow}%
\usepackage{amsmath,amssymb,amsfonts}%
\usepackage{mathrsfs}%
\usepackage[title]{appendix}%
\usepackage{xcolor}%
\usepackage{textcomp}%
\usepackage{manyfoot}%
\usepackage{booktabs}%
\usepackage{algorithm}%
\usepackage{algorithmicx}%
\usepackage{algpseudocode}%
\usepackage{listings}%
\usepackage{todonotes}
\usepackage{url}
\newcommand{\Otilde}{\tilde{O}}

\usepackage{makecell}

\begin{document}
\mainmatter              
\title{Deep Distance Sensitivity Oracles}
\titlerunning{Deep Distance Sensitivity Oracles}  
%
\author{Davin Jeong\inst{1} \and Allison Gunby-Mann\inst{2} \and Sarel Cohen\inst{3} \and Maximilian Katzmann \inst{4} \and Chau Pham \inst{5} \and Arnav Bhakta \inst{6} \and Tobias Friedrich\inst{3} \and Peter Chin\inst{2} }
\authorrunning{Davin Jeong et al.} 

\institute{Harvard University, Cambridge, MA 02138, \\
\email{djeong@college.harvard.edu} 
\and
Dartmouth College, Hanover, NH 03755,\\
\email{allisonmann.th@dartmouth.edu, peter.chin@dartmouth.edu}
\and 
Hasso Plattner Institute, Potsdam, Germany,  \\
\and 
Karlsruhe Institute of Technology, Karlsruhe, Germany
\and
Boston University, Boston MA, 02215
\and 
Yale University, New Haven, CT, 06520
}

\maketitle              

\begin{abstract}

Shortest path computation is one of the most fundamental and well-studied problems in algorithmic graph theory, though it becomes more complex when graph components are susceptible to failure. This research utilizes a \emph{Distance Sensitivity Oracle (DSO)} for efficiently querying replacement paths in graphs with potential failures to avoid inefficiently recomputing them after every outage with traditional techniques. By leveraging technologies such as node2vec, graph attention networks, and multi-layer perceptrons, the study pioneers a method to identify pivot nodes that lead to replacement paths closely resembling optimal solutions with deep learning. Tests on real-world network demonstrate replacement paths that are longer by merely a few percentages compared to the optimal solution. 

\end{abstract}

\section{Introduction}

The shortest path problem is frequently encountered in the real-world. In road networks, users want to know how long it will take to get from one place to another \cite{huang2021learning}. In biological networks, consisting of genes and their products, the shortest paths are used to find clusters and identify core pathways \cite{ren2018shortest}. In social networks, the number of connections between users can be used for friend recommendation \cite{tian2012shortest}. In web search, relevant web pages can be ranked by their distances from queried terms \cite{ukkonen2008searching}. 

For graphs in the real world, often consisting of millions of nodes, special data structures called {\em Distance Oracles} (DO) are used to store information about distances of an input graph $G=(V,E)$ with $n$ vertices and $m$ edges. Without storing the entire graph, they can quickly retrieve important distance information to answer the shortest path queries. 
These shift the computational burden to the preprocessing step, so that queries can be answered quickly.

However, in addition to being large in size, real-world networks are also frequently susceptible to failures. For example, in road networks, a construction, a traffic accident, or an event might temporarily block nodes. In social networks, users might temporarily deactivate their accounts, resulting in a null node. And on the internet, web servers may be temporarily down due to mechanical failures or malicious attacks \cite{billand2016network}. In these instances, we desire a method that can continue answering shortest path queries without stalling or having to recompute shortest paths on the entire graph again.

{\em Distance Sensitivity Oracles} (DSO) are a type of DO that can respond to queries of the form $(s, t, f)$, requesting the shortest path between nodes $s$ and $t$ when a vertex $f$ fails and is thus unavailable. Desirable DSOs should provide reasonable trade-offs among space consumption, query time, and {\em MRE} (i.e., quality of the estimated distance). In this paper, we consider the simplest case, in which there is only one failed node.

\subsection{Contributions}

We have tested our method on a variety of real-world networks and achieved state-of-the-art performance on all of them, our accuracy even outperforms models for shortest paths without node failures. Our contributions mainly lie in three aspects: 
\begin{itemize}
\item In our theoretical analysis, we first present a simple proof for the existence of an underlying combinatorial structure for replacement paths: specifically, the existence of pivot nodes.
\item We observe that one can use deep learning to find pivot nodes in distance sensitivity oracles. In fact, to the best of our knowledge, we are the first to use deep learning to build a distance sensitivity oracle. 
\item We empirically evaluate our method and compare it with related works to demonstrate near-exact accuracy across a diverse range of real-world networks.

\end{itemize}

\subsection{Related Work}

Given we are the first to propose a deep learning approach to DSOs, we describe previous works in both DSOs and deep learning in this section. 

\subsubsection{Distance-Sensitivity Oracles}

The problem of constructing a DSO is well-studied in the theoretical computer science community. Demetrescu {\sl et al.} \cite{DemetrescuTCR08} showed that given a graph $G = (V, E)$, there is a DSO which occupies $O(n^2 \log n)$ space, and can answer a query in constant time. The preprocessing of this DSO, that is the time it takes to construct this DSO, is $O(mn^2+n^3 \log n)$.  

Several theoretical results attempted to improve the preprocessing time required by the DSO. Bernstein and Karger \cite{BeKa09} improved this time bound to $\Otilde(mn)$ \footnote{%
	For a non-negative function $f = f(n)$,
	we use $\Otilde(f)$ to denote $O(f \cdot \textsf{polylog}(n))$.
}. 
Note that the All-Pairs Shortest Paths (APSP) problem, which only asks the distances between each pair of vertices $u, v$, is conjectured to require $mn^{1-o(1)}$ time to solve \cite{LincolnWW18}. Since we can solve the APSP problem by using a DSO, by 
querying it with (s, t, emptyset) for every $s, t$, the preprocessing time $\Otilde(mn)$ is theoretically asymptotically optimal in this sense, up to a polylogarithmic factor (note that, in practice, such polylogarithmic factors may be very large). 
Several additional results improved upon the theoretical preprocessing time by using fast matrix multiplication \cite{ChechikC20,Ren22,GuR21}.

With respect to the size of the oracle, Duan and Zhang \cite{DuZh17} improved the space complexity of \cite{DemetrescuTCR08} to $O(n^2)$, which is from a theoretical perspective asymptotically optimal for dense graphs $(i.e., m = \Theta(n^2))$. To do so, Duan and Zhang store multiple data-structures, which is reasonable for a theoretical work, however from a practical perspective the hidden constant is large. Therefore, it may also be interesting to consider DSOs with smaller space, at the cost of an approximate answer.

Here are several DSOs that provide tradeoffs between the size of DSO and the stretch (the length reported divided by the actual length):

\begin{itemize}
\item The DSO described in \cite{BaswanaK13}, for every parameter $\epsilon > 0 $ and integer $k \ge 1$ has stretch $(2k-1)(1+\epsilon)$ and size $O(k^5n^{1+1/k}\log^3 n/\epsilon^4)$.  
\item The DSO described in \cite{CLPR10}, for every integer parameter $k \ge 1$ has stretch $(16k-4)$ and size $O(kn^{1+1/k} \log n)$. 
\end{itemize}

Note that even though the size of the above two DSOs for $k \ge 2$ is asymptotically smaller than $O(n^2)$, the stretch guarantee is at least $3$ in \cite{BaswanaK13} and at least $28$ in \cite{CLPR10}, which is far from the optimum and may not be practical in many applications. 

In this work, we construct the first DSO that is built using deep learning. Our method uses deep learning to find pivot nodes (as described in Section~\ref{sec:approach}), utilizing a combintorial structural property we observe in Section~\ref{sec:theory}, computing near optimal paths as shown in Section~\ref{sec:results}.

\subsubsection{Shortest Paths using Deep Learning}

Previous works towards answering shortest path queries typically employ a two-stage solution: 1) representation learning and 2) distance prediction. 

In general, graph embeddings are used to generate low-dimensional representations of nodes and edges which preserve properties of the original graph \cite{cai2018comprehensive}. Methods include  matrix factorization, deep learning with and without random walks, edge reconstruction, graph kernels, and generative models \cite{cai2018comprehensive}. By either optimizing embeddings for their specific task or using general embedding techniques like node2vec, these graph embeddings may then be combined with existing techniques to tackle tasks such as node classification, node clustering, and link detection \cite{zhang2021graphnet,crichton2018neural}.

Among the first to apply graph embeddings to the shortest paths problem was Orion \cite{wang2002orion}. Inspired by the successes of virtual coordinate systems, a landmark labelling approach was employed, where positions of all nodes were chosen based on their relative distances to a fixed number of landmarks. Using the Simplex Downhill algorithm, representations were found in a Euclidean coordinate space, allowing constant time distance calculations and producing mean relative error (MRE) between 15\% - 20\% \cite{wang2002orion}. Other existing coordinate systems have also been used. Building off of network routing schemes in hyperbolic spaces, Rigel used a hyperbolic graph coordinate system to reduce the MRE to 9\% and found that the hyperbolic space performed empirically better across distortion metrics than Euclidean and spherical coordinate systems \cite{cvetkovski2009hyperbolic,zhao2011efficient}. In road networks, geographical coordinates have been utilized with a multi-layer perceptron to predict distances between locations with 9\% MRE \cite{jindal2017unified}. 

In addition to these coordinate systems, general graph embedding techniques have recently been employed to handle shortest path queries to great success. In 2018, researchers from the University of Passau proposed node2vec-Sg\cite{rizi2018shortest}. To find the shortest path between nodes $s$ and $t$, their Node2vec and Poincare embeddings were combined through various binary operations and fed into a feed-forward neural network, which was trained only on the distances between $l$ landmark nodes $l<<n$ and the rest of the graph. The model which took concatenated Node2vec embeddings performed the best, with an MRE between 3\% to 7\% .

Researchers have also demonstrated the accuracy of graph embeddings learned alongside distance predictors, to produce representations more specific to the shortest path task. Vdist2vec directly learned vertex embeddings by passing the gradient from the distance predictor back to a $N \times k$ matrix, achieving an MRE between 1\% to 7\% \cite{qi2020learning}. Huang et al. computed shortest path distances on road networks using a hierarchical embedding model and achieved an MRE of 0.7\% \cite{huang2021learning}. Most recently, ndist2vec built upon the landmark learning, graph embedding, and neural network aspects of all of these approaches, reporting an MRE of $3.4\%$ with a dataset on the order of $O(n)$. 

Current works for estimating the shortest path lengths between two nodes are limited by the representations they learn. They rely on datasets which, even using schemas like landmark labelling or hierarchical, are proportional to $n$, the number of nodes in the network\cite{chen2022ndist2vec,huang2021learning}. This presents a significant bottleneck for larger graphs. Taking these lessons to the deep learning DSO task, we present a model in Section~\ref{sec:approach}, which extracts signal more efficiently thus requiring training samples without sacrificing accuracy. 

\section{Theoretical Analysis}
\label{sec:theory}
In this section we consider a combinatorial structural property of replacement paths: such a path is a concatenation of a few original shortest paths. As we previously described, later on our deep learning algorithm builds on this lemma.

Our research is motivated by the following lemma from Afek {\sl et al.}.

\begin{lemma}\cite{MPLS}
\label{lem:k-fails}
After $k$ edge failures in an unweighted graph, each new shortest path is the concatenation of at most $k + 1$ original shortest paths.
\end{lemma}

In other words, the replacement path can be defined using so called pivot nodes that specify at which nodes in the graph the shortest paths may be stitched together.  In this work we are interested in the failure of a single node, which is equivalent to the failure of its incident edges.  The number of concatenations (and with that the number of corresponding pivots) required to obtain the replacement path then depends on the degree of the failed node.  While in real-world networks the average degree is often rather small, finding suitable pivots remains a hard task.  To overcome this problem, we consider an approximate setting, where we allow for a slack in the quality of the obtained paths (they may be longer than a shortest replacement path) but where only one pivot node is used. From the theoretical perspective, the following lemma 
is a special case of Lemma~\ref{lem:k-fails} for the case of a single edge failure.

\begin{lemma} \label{lemma:pivot_node}
After an edge failure in an unweighted undirected graph, each new shortest path is the concatenation of at most two original shortest paths.
\end{lemma}

Given $(s,t,f)$, let $P(s,t,f)$ be a shortest path from $s$ to $t$ in $G-\{f\}$. According to Lemma \ref{lemma:pivot_node} it follows that $P(s,t,f)$ is a concatenation of two original shortest paths, or in other words, there exists a pivot vertex $v$ such that $P(s,t,f)$ is the concatenation of the two shortest paths $P(s,v)$ and $P(v,t)$, where $P(s,v)$ is a shortest path from $s$ to $v$ in $G$ and $P(v,t)$ is a shortest path from $v$ to $t$ in $G$. 
Motivated by this lemma, we will assume that it is sufficient to approximate the replacement path of a node failure using a single pivot node. As mentioned above, in the remainder of this paper we show that it is possible to use deep learning to find such pivot nodes.

\section{Method} 
\label{sec:approach}

\begin{figure*}[h]
\centering
\includegraphics[width=0.9\textwidth]{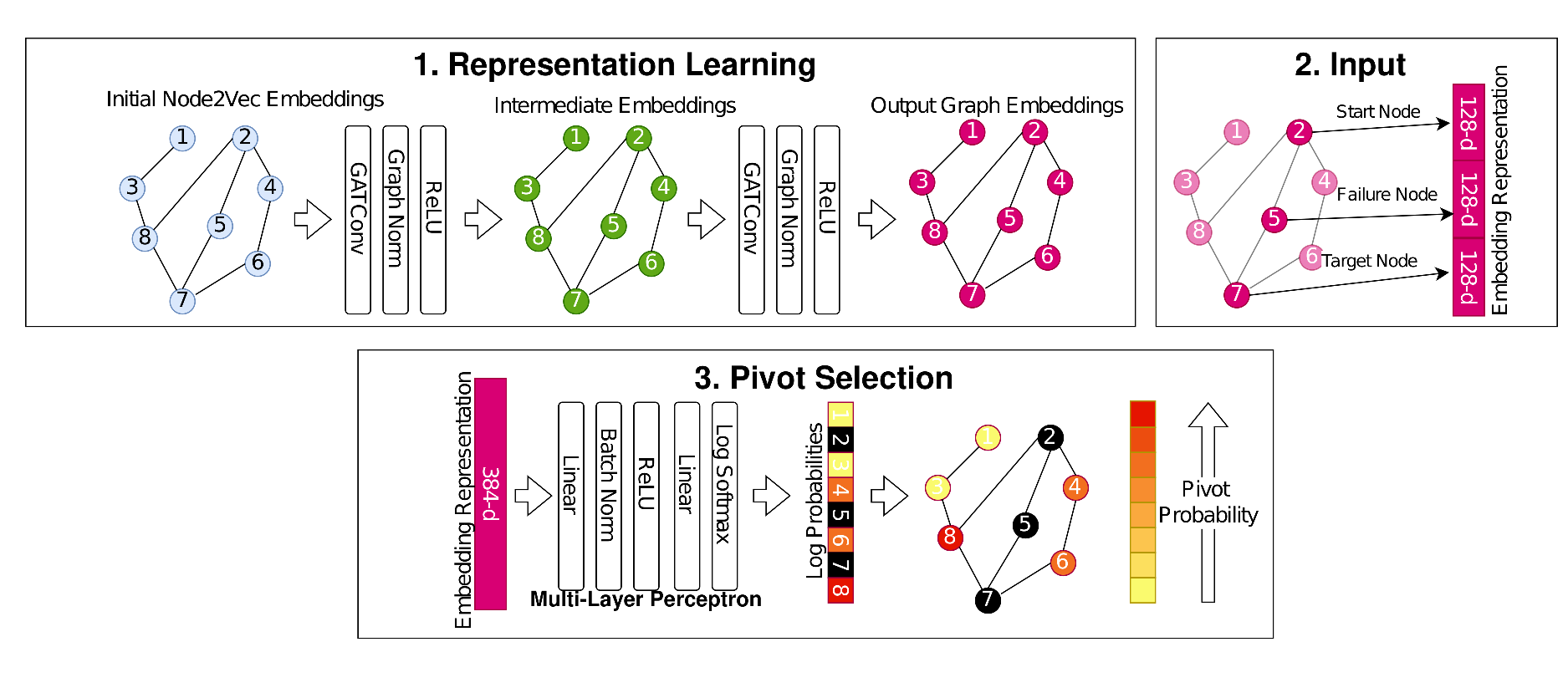}
\caption{The overall neural network architecture. During back-propagation, the gradient is passed through the MLP and to the relevant GAT parameters, so that the representations learned to encode relevant task-specific attributes.}
\label{fig:modelarchitecture}
\end{figure*}

By the above argumentation, we reduce the problem of finding the shortest replacement path to finding pivot candidates. In the following section, we describe how we use a graph convolutional network to encode relevant graph information and a multi-layer perceptron to select pivot nodes. 

\subsection{Training Data}

To generate our training data, we first generated training triplets $(s, t, f)$, representing start, target, and failed nodes. $s$ and $t$ were rejected if not connected or if direct neighbors. For each input triplet, we simulated the node failure, computed replacement paths, and determined the pivot nodes. 

$min(10^5, n^2)$ input triplets were randomly sampled for a graph with $n$ nodes (Section~\ref{table:datasets}). An 80/10/10 split was used for training, validation, and testing sets. Though previous works using landmark labelling selected $l(n-l)$ samples to improve the quality of their candidates, we found that capping our dataset by a constant value was sufficient \cite{huang2021learning,rizi2018shortest}. This value can be modified based on the input graph if desired. 

\subsection{Representation Learning with Graph Convolutional Networks}

In order to learn representations from not only the shortest path data but also the network structure, we leverage the Graph Attention (GAT) layer. This attention-based convolutional layer computes each nodes' embeddings by implicitly learning different importance to different nodes within a neighborhood \cite{velivckovic2017graph}. While they have been used successfully for tasks like biological link prediction and text classification, to the best of our knowledge, we are the first to apply the GAT mechanism to shortest path problems. 

Like previous deep learning for shortest paths works, we initialize our node embeddings using node2vec\cite{grover2016node2vec,rizi2018shortest}. We use 16 iterations, a walk length of 10, num walks of 80, $p = 1$, $q = 1$, and window size of 10. Next, we use two GAT layers to refine our embeddings towards our task. We introduce non-linearities using a ReLU layer, and we speed up convergence using GraphNorm layers, a normalization technique that was recently found to be more effective than similar techniques for training graph neural networks \cite{cai2021graphnorm}. We also use dropout (0.1) to deal with overfitting. 

\subsection{Multi-Layer Perceptron}

Multiple pivot nodes are possible, so we frame the task of pivot selection as a multi-layer classification task. As shown in Figure \ref{fig:modelarchitecture}, after the GAT component of our model, we select the three embeddings corresponding the start, target, and failing node triplet. These are flattened into a 384-sized vector, which serves as the input for our feedforward neural network. 

The multilayer perceptron consists of two linear layers and an output layer. We use ReLU layers as an activation function between the first two layers to introduce non-linearities. We also use BatchNorm layers to enable faster convergence and introduce regularization. The output layer is a log-softmax layer (commonly used for multi-label classification) that normalizes the vector values. Our final output is a $n$-sized vector, representing the log likelihoods of each node in the network being a pivot node. 

\subsection{Summary}

In total, our neural network consists of a representation learning module, based both on characteristics of the network and our DSO task, and a pivot selection module, a simple MLP (Figure~\ref{fig:modelarchitecture}). Since multiple pivot nodes are possible, the model was trained using BCELoss and the node with the highest likelihood was selected as the output. We trained our neural network for 32 epochs using the Adam optimizer \cite{kingma2014adam}. Our learning rate is 1e-3, our betas are 0.9 and 0.999, and our epsilon value is 1e-8.  

\section{Experiments}
\label{sec:experiments} 

As mentioned previously, to the best of our knowledge, we are the first to propose a deep learning approach to the DSO problem. Thus, we will evaluate our proposed method against the state-of-the-art deep learning approaches to the shortest paths problem: namely, ndist2vec and node2vec-Sg \cite{chen2022ndist2vec,rizi2018shortest}. 

We ran the authors' implementations for all comparison models with their recommended hyperparameter settings. All embeddings were 128-dimensional, in line with previous shortest paths works \cite{grover2016node2vec,qi2020learning,rizi2018shortest}. 

All experiments were implemented in Python and run on a Quadro RTX 8000 and an Intel(R) Xeon(R) Silver 4214R CPU @ 2.40GHz.

\subsection{Datasets}

\begin{table}[]
\centering
\resizebox{0.5\textwidth}{!}{%
\begin{tabular}{lcccc}
\hline
\multicolumn{1}{c}{\textbf{Network Name}} &
  \multicolumn{1}{c}{\textbf{Nodes}} &
  \multicolumn{1}{c}{\textbf{Density}} &
  \multicolumn{1}{c}{\textbf{\thead{Average\\ Degree}}} &
  \multicolumn{1}{c}{\textbf{\thead{Dataset\\ Size}}} \\ \hline
chem-ENZYMES-g118          & 95      & 2.71E-02 & 2.547   & 9.03E+03 \\
chem-ENZYMES-g296          & 125     & 1.82E-02 & 2.256   & 1.56E+04 \\
infect-dublin              & 410     & 3.30E-02 & 13.488  & 1.00E+05 \\
bio-celegans               & 453     & 1.98E-02 & 8.940   & 1.00E+05 \\
bn-mouse-kasthuri-graph-v4 & 987     & 3.16E-03 & 3.112   & 1.00E+05 \\
can-1072                   & 1,072   & 1.18E-02 & 12.608  & 1.00E+05 \\
scc\_retweet               & 1,150   & 9.98E-02 & 114.713 & 1.00E+05 \\
power-bcspwr09             & 1,723   & 2.78E-03 & 4.779   & 1.00E+05 \\
inf-openflights            & 2,905   & 3.71E-03 & 10.771  & 1.00E+05 \\
inf-power                  & 4,941   & 5.40E-04 & 2.669   & 1.00E+05 \\
ca-Erdos992                & 4,991   & 5.97E-04 & 2.977   & 1.00E+05 \\
power-bcspwr10             & 5,300   & 9.66E-04 & 5.121   & 1.00E+05 \\
bio-grid-yeast             & 6,008   & 8.70E-03 & 52.245  & 1.00E+05 \\
soc-gplus                  & 23,613  & 1.41E-04 & 3.319   & 1.00E+05 \\
ia-email-EU                & 32,430  & 1.03E-04 & 3.355   & 1.00E+05 \\
ia-wiki-Talk               & 92,117  & 8.50E-05 & 7.833   & 1.00E+05 \\
dbpedia-occupation         & 127,569 & 3.08E-05 & 3.934   & 1.00E+05 \\
tech-RL-caida              & 190,914 & 3.33E-05 & 6.365   & 1.00E+05 \\ \hline
\end{tabular}%
}
\caption{Statistics of real-world networks.}
\label{table:datasets}
\end{table}

We test our method on a variety of undirected, unweighted networks from the Network Repository~\cite{nr}, which covers a diverse set of areas, such as road networks, biological networks, and communication networks, all representing potential input for real-world applications.  For a list of all considered networks, we refer to Table~\ref{table:datasets}. 

The training schema of each comparison model was used, but the testing dataset was standardized for the purposes of a fair comparison. Specifically, $min(10^5, n^2) * 0.10$ node pairs $(a, b)$ were randomly sampled without replacement from each network, such that $a \neq b$. 

\subsection{Evaluation}

In line with previous works computing shortest paths with deep learning, we evaluate our method using the \textbf{Mean Relative Error (MRE)} metric. Let $\hat{d}_{i, j}$ denote the predicted distance and $d_{i, j}$ denote the actual distance. The Relative Error is then given by 

\begin{equation*}
    RE = \frac{|\hat{d}_{i, j} - d_{i, j}|}{d_{i, j}}
\end{equation*}

We note that, for evaluations of DSOs, $\hat{d}_{i, j}$ and $d_{i, j}$ denote the predicted and actual distances on the graph after a node failure. 

We also provide the representation factor, denoting the ratio of the MRE obtained with random pivots and the one obtained using our method, as a metric for evaluating the quality of our representations.

\section{Results}
\label{sec:results}

\begin{table}[]
\centering
\resizebox{\textwidth}{!}{%
\begin{tabular}{lccccc}
\hline
\textbf{Network Name} &
  \multicolumn{1}{l}{\textbf{\thead{MRE\\ (Ndist2vec)}}} &
  \multicolumn{1}{l}{\textbf{\thead{MRE\\ (Rizi)}}} &
  \multicolumn{1}{l}{\textbf{\thead{MRE\\ (Random Pivots)}}} &
  \multicolumn{1}{l}{\textbf{\thead{Representation\\ Factor}}} &
  \multicolumn{1}{l}{\textbf{\thead{MRE\\ (Our Method)}}} \\ \hline
chem-ENZYMES-g118          & 100.00\%             & 14.19\%  & 215.96\% & 811.89   & 0.27\% \\
chem-ENZYMES-g296          & 100.00\%             & 7.85\%   & 257.69\% & 530.22   & 0.49\% \\
infect-dublin              & 38.33\%              & 13.58\%  & 191.02\% & 1,091.55 & 0.18\% \\
bio-celegans               & 15.45\%              & 9.84\%   & 176.28\% & 4,299.51 & 0.04\% \\
bn-mouse-kasthuri-graph-v4 & 17.15\%              & 16.91\%  & 204.20\% & 5,105.11 & 0.04\% \\
can-1072                   & 36.43\%              & 11.23\%  & 211.58\% & 573.40   & 0.37\% \\
scc\_retweet               & 100.00\%             & 10.78\%  & 167.64\% & 3,287.11 & 0.05\% \\
power-bcspwr09             & 42.75\%              & 20.45\%  & 237.46\% & 1,493.46 & 0.16\% \\
inf-openflights            & 100.00\%             & 16.32\%  & 202.05\% & 280.63   & 0.72\% \\
inf-power                  & 59.08\%              & 31.67\%  & 228.67\% & 1,013.10 & 0.23\% \\
ca-Erdos992                & 100.00\%             & 18.93\%  & 206.26\% & 630.77   & 0.33\% \\
power-bcspwr10             & 37.02\%              & 31.95\%  & 232.83\% & 739.13   & 0.32\% \\
bio-grid-yeast             & 14.67\%              & 15.95\%  & 173.80\% & 27.80    & 6.25\% \\
soc-gplus                  & 100.00\%             & 55.99\%  & 200.20\% & 654.61   & 0.31\% \\
ia-email-EU                & 13.31\%              & 119.61\% & 202.93\% & 216.11   & 0.94\% \\
ia-wiki-Talk               & 20.83\%              & 28.19\%  & 203.29\% & 26.50    & 7.67\% \\
dbpedia-occupation         & \multicolumn{1}{c}{} & 28.41\%  & 205.07\% & 31.56    & 6.50\% \\
tech-RL-caida &
  \multicolumn{1}{c}{\multirow{-2}{*}{Did not finish}} &
  51.92\% &
  206.02\% &
  26.04 &
  7.91\% \\ \hline
\end{tabular}%
}
\caption{Results of models across several real-life
networks.}
\label{table:results}
\end{table}

We aimed to evaluate our deep learning approach across the following questions:

\begin{enumerate}
    \item How much longer than the optimal paths are our replacement paths?
    \item How does our deep learning model compare with previous state-of-the-art shortest paths works?
    \item Is the resulting performance an achievement of our approach or merely an artifact of the structure of the input graph?
\end{enumerate}

The motivation behind the first question is obvious. Computed replacement paths are only suitable if they are not much longer than the actual shortest paths in the graph. Table~\ref{table:results} lists the MRE values obtained on all considered networks. Except for \textbf{bio-grid-yeast}, \textbf{ia-wiki-Talk}, \textbf{tech-RL-caida}, and \textbf{dbpedia-occupation} all MRE values are below $1\%$, with a mean value of $1.82\%$ and a median value of $0.32\%$. This means that the computed replacement paths were almost as short as the optimal replacement paths. Interestingly, the networks with the highest MRE values were not those with the highest density or average degree. For instance, our model obtained an MRE less than $1\%$ \textbf{soc-gplus} and \textbf{ia-email-EU} but struggled with the comparably smaller \textbf{bio-grid-yeast} network at $6.25\%$. These insights suggest that other factors related to the networks' structures affected our models' performance

Given the lack of DSOs using deep learning, the second question hopes to understand how our model performs compared to methods finding the shortest paths without node failures. Across all networks, we were able to match or outperform the state-of-the-art shortest paths works, often by several degrees of magnitude (Table~\ref{table:results}). We did so with less training cases (numerically and proportionally) and no special selection process. In doing so, we demonstrated that deep learning can be used to effectively find the shortest replacement paths as well. 

We'd like to note that we used the authors' implementation of ndist2vec and confirmed its performance on the road network datasets presented in the original paper \cite{chen2022ndist2vec}. Nonetheless, the model performed significantly worse on our real-world networks and did not finish after a week of computation for the two largest networks. We have two potential explanations. First, the authors suggest that their landmark-labelling approach will not scale well to sparse, large networks, many of which were used in our experiments. Second, the model often became stuck on local optima during training, producing a MRE value of $100\%$ corresponding to a constant ouput of $0$, demonstrating a reliance on initial values. For a fair comparison, we depict the best MRE values after two runs in Table~\ref{table:results}.

Finally, we aimed to determine whether our model performed well because of the model or because of the inherent structure of the networks. For example, consider a graph that is almost a clique (almost all vertices are pairwise connected).  Then, all paths are short and after a failure (having almost no impact on the graph structure) most replacement paths are short as well. In this setting almost any node can serve as a suitable pivot, yielding a replacement path with a small stretch and one would expect that even randomly chosen pivots would yield good results.  While the networks considered in our experiments are not as dense (see Table~\ref{table:results}), other graph properties like a small diameter may make finding good pivots easier.

Table~\ref{table:results} lists the MRE values we get when considering random pivots, which we obtained by replacing the output of our pipeline with random noise. As can be clearly seen, the MRE is much larger in this setting.  For most networks the MRE is larger than $200\%$, meaning the found paths are more than $3$ times longer than the shortest replacement paths.  In order to compare our method with the random approach, Table~\ref{table:results} also lists the representation factor.  Except for \textbf{bio-grid-yeast}, \textbf{ia-wiki-Talk}, and \textbf{tech-RL-caida} this factor is always larger than $200$, meaning on most networks our approach is over $200$ times better than the random method, clearly indicating that the close to optimal performance is due to the quality of our approach and not an artifact of properties of the considered inputs.

\section{Conclusion} 

We have shown that distance sensitivity oracles with close to optimal performance can be obtained by utilizing the power of deep learning.  Our method builds on a combinatorial property that allows for finding replacement paths based on pivot vertices.  On a variety of real-world networks in the presence of failures, we can reliably find suitable pivots where the lengths of the corresponding replacement paths are very close to those of optimal paths.  Moreover, our experiments suggest that these results are not artifacts of the inherent structure of the inputs, but are instead based on the fact that the different building blocks of our pipeline successfully capture the relevant structural information about the input graph.

As a consequence, it would be interesting to apply this method to related tasks where similar structural information needs to be captured.  One such example is \emph{local routing}, where the goal is to find short paths in a graph without the use of a central data structure by greedily routing to nearby embeddings. Prior work has shown that close to optimal greedy routing can be performed when embedding networks into hyperbolic space~\cite{blasius2020greedy}.  However, the resulting embeddings were susceptible to numerical inaccuracies, and network failures decreased routing performance a lot.  It would thus be interesting to see whether our approach can be extended to the greedy routing setting as well, in order to overcome the previously observed issues.

Additionally, our approach has currently not been tested on larger networks containing millions of nodes. By calculating the APSP information using a distance oracle and using an improved node2vec implementation, we plan to test our networks' scalability in the future. 

%
%

\bibliographystyle{spmpsci} 
\bibliography{my_bib} 

\end{document}